\documentclass{article}
\usepackage{journal}
\usepackage[mathscr]{eucal}
\usepackage[usenames,dvipsnames]{xcolor}
\usepackage{natbib}
\usepackage[backref,pageanchor=true,plainpages=false,pdfpagelabels,bookmarks,bookmarksnumbered,breaklinks,pdfborder={0 0 0},colorlinks=true,citecolor=Sepia,linkcolor=Sepia,filecolor=Sepia,urlcolor=Sepia]{hyperref}

\usepackage{framed}
\usepackage{amsmath,amsfonts,amsbsy,amsgen,amscd,mathrsfs,amssymb}
\usepackage[colorinlistoftodos]{todonotes}
\usepackage[utf8]{inputenc} 
\usepackage[T1]{fontenc}    
\usepackage{hyperref}       
\usepackage{booktabs}       
\usepackage{nicefrac}       
\usepackage{microtype}      
\usepackage[framemethod=TikZ]{mdframed}
\mdfsetup{skipabove=\topskip,skipbelow=\topskip}

\usepackage{bm} 
\usepackage{bbm}
\usepackage{xspace}
\usepackage[algo2e,ruled,linesnumbered,vlined]{algorithm2e}
\usepackage{tcolorbox}

\usepackage{amsthm}
\usepackage{mathtools}
\usepackage[font=small,margin=0.25in,labelfont={sc},labelsep={colon}]{caption}
\usepackage[shortlabels]{enumitem}

\makeatletter
\def\set@curr@file#1{\def\@curr@file{#1}} 
\makeatother
\usepackage[load-configurations=version-1]{siunitx} 

\usepackage{accents}

\newtheorem{thm}{Theorem}
\newtheorem{lem}[thm]{Lemma}
\newtheorem{rem}{Remark}

\theoremstyle{definition}
\newtheorem{defn}{Definition}
\newtheorem*{defn*}{Definition}
\renewenvironment{proof}[1][]{\par\noindent{\bf Proof #1\ }}{\hfill\BlackBox\\[2mm]}

\usepackage{prettyref}
\newrefformat{fig}{Figure~\ref{#1}}
\newrefformat{alg}{Algorithm~\ref{#1}}
\newrefformat{def}{Definition~\ref{#1}}
\newrefformat{eqn}{Equation~\ref{#1}}
\newrefformat{cor}{Corollary~\ref{#1}}
\newrefformat{app}{Appendix~\ref{#1}}
\newrefformat{rem}{Remark~\ref{#1}}
\newrefformat{assu}{Assumption~\ref{#1}}

\jmlrheading{}{}{}{}{}{}
\ShortHeadings{}{}
\newcommand{\inbrace}[1]{\left \{ #1 \right \}}
\newcommand{\inparen}[1]{\left ( #1 \right )}
\newcommand{\insquare}[1]{\left [ #1 \right ]}

\newcommand{\inabs}[1]{\begin{vmatrix} #1 \end{vmatrix}}

\newcommand{\abs}[1]{\left\lvert #1 \right\rvert}

\newlength{\dhatheight}

\newcommand{\ceil}[1]{\left \lceil #1 \right \rceil}
\newcommand{\floor}[1]{\left \lfloor #1 \right \rfloor}

\newcommand{\set}[1]{\inbrace{#1}}

\DeclareMathOperator*{\argmin}{argmin}

\DeclareMathOperator*{\Ex}{\mathbb{E}}

\DeclareMathOperator*{\Prob}{Pr}

\renewcommand{\Pr}{\mathbf{Pr}}

\newcommand{\bbN}{{\mathbb N}}

\newcommand{\bbA}{{\mathbb A}}

\let\boldm\bm

\newcommand{\bx}{{\boldm x}}
\newcommand{\by}{{\boldm y}}
\newcommand{\bz}{{\boldm z}}
\newcommand{\tbx}{\Tilde{{\boldm x}}}
\newcommand{\tby}{\Tilde{{\boldm y}}}
\newcommand{\tbz}{\Tilde{{\boldm z}}}

\newcommand{\calD}{\mathcal{D}}

\newcommand{\calH}{\mathcal{H}}

\newcommand{\calU}{\mathcal{U}}

\newcommand{\calX}{\mathcal{X}}
\newcommand{\calY}{\mathcal{Y}}

\newcommand{\B}{\mathrm{B}}

\newcommand{\ERM}{\textsf{ERM}\xspace}

\newcommand{\ind}{\mathbbm{1}}
\newcommand{\Risk}{{\rm R}}
\newcommand{\TRisk}{{\rm TR}}
\newcommand{\IRisk}{{\rm IR}}

\newcommand{\err}{{\rm err}}
\newcommand{\vc}{{\rm vc}}

\newcommand{\Rdim}{{\rm rdim}}

\newcommand{\OPT}{\mathsf{OPT}}

\newcommand{\removed}[1]{}

\begin{document}

\title{Transductive Robust Learning Guarantees}
\author{%
 \name{Omar Montasser} \email{omar@ttic.edu}\\
 \addr Toyota Technological Institute at Chicago\\
 \name{Steve Hanneke} \email{steve.hanneke@gmail.com}\\
 \addr Purdue University\\
 \name{Nathan Srebro} \email{nati@ttic.edu}\\
 \addr Toyota Technological Institute at Chicago
}

\maketitle

\begin{abstract}%
We study the problem of adversarially robust learning in the transductive setting. For classes $\calH$ of bounded VC dimension, we propose a simple transductive learner that when presented with a set of labeled training examples and a set of unlabeled test examples (both sets possibly adversarially perturbed), it correctly labels the test examples with a robust error rate that is linear in the VC dimension and is adaptive to the complexity of the perturbation set. This result provides an exponential improvement in dependence on VC dimension over the best known upper bound on the robust error in the inductive setting, at the expense of competing with a more restrictive notion of optimal robust error.
\end{abstract}

\section{Introduction}
\label{sec:intro}

We consider the problem of learning predictors that are {\em robust} to adversarial examples at test time. That is, we would like to be robust against a perturbation set $\calU:\calX \to 2^{\calX}$, where $\calU(x)\subseteq \calX$ is the set of allowed perturbations that an adversary might replace $x$ with, as measured by the {\em robust risk}:
\begin{equation}
    \label{eqn:rob-risk}
\Risk_{\calU}(h;\calD) =\Ex_{(x,y)\sim\calD} \insquare{\sup_{z\in \calU(x)}\ind\set{h(z)\neq y}}.
\end{equation}
For example, $\calU$ could be perturbations of bounded $\ell_p$-norms \citep{DBLP:journals/corr/GoodfellowSS14}.

Adversarially robust learning has been studied almost exclusively in the \emph{inductive} setting, where the task is to learn, from (non-adversarial) training data, a \emph{predictor} with small robust risk (\prettyref{eqn:rob-risk}) \citep{pmlr-v99-montasser19a}. In many applications in practice, however, test examples are available in batches and machine learning systems are tasked with classifying them all at once. \emph{Transductive} learning refers to the learning setting where the goal is to classify a given unlabeled test set that is presented together with the training set \citep*{vapnik1998statistical}. 

In this paper, we study adversarially robust learning in the \emph{transductive} setting. In this problem, $n$ i.i.d.~training examples $(\bx,\by)\sim\calD^n$ and $m$ separate i.i.d.~test examples $(\tbx,\tby)\sim\calD^m$ are drawn from some unknown distribution $\calD$. Then, based on all available information: $\bx,\by,\tbx,\tby$, distribution $\calD$, perturbation set $\calU$, and white-box access to the transductive learner $\bbA:(\calX\times\calY)^n\times \calX^m \to \calY^m$, an adversary chooses adversarial perturbations of the test set $\Tilde{z}_i\in \calU(\Tilde{x}_i)\forall i \in [m]$, which we henceforth denote by $\tbz\in\calU(\tbx)$. Finally, the transductive learner $\bbA$ receives as input the labeled training examples $(\bx,\by)$ and the perturbed test examples $\tbz$, and outputs a labeling for $\tbz$ which we denote by $\bbA(\bx,\by,\tbz) \in \calY^m$\footnote{Throughout the paper, we abuse notation and use $\bbA(\bx,\by,\tbz)(\Tilde{z}_i)$ to refer to the $i^{\text{th}}$ entry in the vector $\bbA(\bx,\by,\tbz)$.}. The performance of $\bbA$ is measured by the \emph{transductive robust risk}\footnote{Unless otherwise stated, in this paper we fix the test set size $m=n$.}:
\begin{equation}
    \label{eqn:rob-risk-transductive}
    \TRisk^{n,m}_{\calU}(\bbA;\calD) = \underset{\substack{(\bx,\by)\sim \calD^n\\(\Tilde{\bx},\Tilde{\by})\sim\calD^m}}{\Ex}\insquare{\sup_{\tbz\in\calU(\tbx)} \frac{1}{m} \sum_{i=1}^{m} \ind\set{\bbA(\bx,\by,\tbz)(\Tilde{z}_i) \neq \Tilde{y}_i} }.
\end{equation}

As we shall show, the transductive setting allows for much stronger results than what is known in the inductive adversarially robust setting.

How is this possible?  In traditional (non-robust) learning, there are standard transductive-to-inductive and inductive-to-transductive reductions which establish that both settings are essentially equivalent statistically.  However, in \prettyref{sec:comparisons} we discuss how the inductive-to-transductive reduction breaks down for adverserially robust learning, opening the possibility that transductive robust learning might be inherently easier than inductive robust learning.  This is good news, since inductive adversarially robust learning so far seems challenging. Focusing on adversarially robust learning of VC classes (hypothesis classes with bounded VC dimension), although we know all such classes are adverserially robustly learnable, existing inductive methods require sample complexity exponential in the VC dimension, and a completely intractable and essentially non-implementable algorithm \citep{pmlr-v99-montasser19a}. In contrast, for transductive adversarially robust learning, we present a simple and straight-forward algorithm with sample complexity only \emph{linear} in the VC-dimension!

So why are we interested in the transductive setting?  First, if the adversarially robust transductive setting is indeed easier than its inductive counterpart, it is important to develop methods that take advantage of this setting, and could be applicable and beneficial when entire batches of test examples are processed concurrently.  This paper is the first work, as far as we are aware, in this direction.  Alternatively, perhaps advances in analyzing the transductive setting could potentially translate back to the inductive setting---although the standard reduction does not apply, we can still be hopeful we might close the gap through additional ideas.

\paragraph{Relaxed guarantees: choice of competitor} As with most learning theory gurantees, we will show how, given enough samples, we can approach the error of some reference competitor.  The best we can hope for is to compete with $\OPT_\calU=\inf_{h\in\calH} \Prob_{(x,y)\sim\calD} \insquare{\exists z\in \calU(x): h(z)\neq y}$, which is the smallest attainable robust risk against perturbation set $\calU$---this is the best we could do even if we knew the source distribution. In this work, we consider a weaker goal where we compete with the smallest attainable robust risk against a stronger adversary:
\begin{equation}
\label{eqn:opt}
\OPT_{\calU^{-1}(\calU)}=\inf_{h\in\calH} \Prob_{(x,y)\sim\calD} \insquare{\exists \Tilde{x}\in\calU^{-1}(\calU)(x): h(\Tilde{x})\neq y},
\end{equation}
where $\calU^{-1}(z)=\set{x\in\calX: z\in\calU(x)}$ and $\calU^{-1}(\calU)(x)=\cup_{z\in\calU(x)} \calU^{-1}(z)=\set{\Tilde{x}\in\calX: \calU(x)\cap \calU(\Tilde{x})\neq \emptyset}$.

In words, $\OPT_{\calU^{-1}(\calU)}$ is the smallest attainable robust risk against the larger perturbation set $\calU^{-1}(\calU)$. In particular, when $x\in\calU(x)$, $\calU(x)\subseteq \calU^{-1}(\calU)(x)$ and $\OPT_\calU\leq \OPT_{\calU^{-1}(\calU)}$. And what we will show is a transductive learner $\bbA$ with robust risk $\TRisk_\calU(\bbA;\calD)$ which is competitive with the best robust risk $\OPT_{\calU^{-1}(\calU)}$ against the larger perturbation set $\calU^{-1}(\calU)$.

For example, consider $\calU(x)=\B_\gamma(x)\triangleq \set{z \in \calX: \rho(x,z)\leq \gamma}$ where $\gamma>0$ and $\rho$ is some metric on $\calX$ (e.g., $\ell_p$-balls). In this case, $\calU^{-1}(\calU)(x)=\B_{2\gamma}(x)$. Furthermore, $\OPT_\calU$ corresponds to optimal robust risk with radius $\gamma$, while $\OPT_{\calU^{-1}(\calU)}$ corresponds to optimal robust risk with radius $2\gamma$. In this case, our guarantees will ensure robustness to perturbations within radius $\gamma$, that is almost as good as the best possible robust risk with radius $2\gamma$. In particular, our guarantees in the realizable setting ensure robustness to perturbations within radius $\gamma$ when the smallest robust risk with radius $2\gamma$ is zero, i.e.,$\OPT_{\calU^{-1}(\calU)}=0$. By way of analogy, guarantees that are similar in spirit are common in the context of bi-criteria approximation algorithms for discrete optimization problems (e.g., the sparsest cut approximation algorithm due to \citet{DBLP:conf/stoc/AroraRV04}).

\paragraph{Main Contributions} We shed some new light on the problem of adversarially robust learning by studying the transductive setting. We propose a \emph{simple} transductive learning algorithm with robust learning guarantees that are stronger than the known inductive guarantees in some aspects, but weaker in other aspects. Specifically, our algorithm enjoys an improved robust error rate that is at most \emph{linear} in the VC dimension and is adaptive to the complexity of the perturbation set $\calU$, and is also robust to adversarial perturbations in the \emph{training} data. This comes at the expense of competing with the more restrictive $\OPT_{\calU^{-1}(\calU)}$, where the inductive guarantees compete with $\OPT_\calU$. 

Specifically, given a class $\calH$ and a perturbation set $\calU$, we present a simple tansductive learner $\bbA:(\calX\times \calY)^n\times \calX^n\to \calY^n$ (see \prettyref{sec:main-results}) such that for any distribution $\calD$ over $\calX\times\calY$:
\begin{align}
\label{eqn:transductive-guarantees}
\TRisk_\calU(\bbA;\calD) &\leq \frac{\vc(\calH)\log(2n)}{n} & &\textsf{(Realizable, $\OPT_{\calU^{-1}(\calU)}=0$)},\\
\TRisk_\calU(\bbA;\calD) &\leq 2\OPT_{\calU^{-1}(\calU)} + O\inparen{\sqrt{\frac{\vc(\calH)}{n}}} & &\textsf{(Agnostic, $\OPT_{\calU^{-1}(\calU)}>0$)}.
\end{align}

Our transductive learner $\bbA$ simply asks for any predictor $h\in \calH$ that robustly and correctly labels the training examples $(\bx,\by)$ with respect to $\calU^{-1}$ and robustly labels the test examples $\bz$ with respect to $\calU^{-1}$. In \prettyref{sec:main-results}, we show that our transductive learner additionally enjoys the following properties:
\begin{enumerate}
    \item Robustness guarantees against adversarial perturbations in the \emph{training} data. These are the first non-trivial learning guarantees against adversarial perturbations in the training data, which has not been considered before in the literature to the best of our knowledge.
    \item Adaptive robust error rates that are controlled by the complexity of $\calH$ and the perturbation set $\calU$ in the form of a new complexity measure that we introduce: the \emph{relaxed $\calU$-robust shattering dimension} $\Rdim_{\calU}(\calH)$ (see \prettyref{def:robustshatter-dim-relax}). These are the first general robust learning guarantees that take the complexity of the perturbation set $\calU$ into account.  
\end{enumerate}

\paragraph{Practical Implications} In the context of deep learning and robustness to $\ell_p$ perturbations, and in scenarios where (adversarial) unlabeled test data is available in batches, our results suggest that to incur a low error rate on the test data it suffices to perform adversarial training \citep[e.g.,][]{DBLP:conf/iclr/MadryMSTV18, DBLP:conf/icml/ZhangYJXGJ19} to find network parameters that simultaneously: (a) robustly and correctly fit the labeled training data, and (b) robustly fit the unlabeled (adversarial) test data. This is in contrast with the inductive setting, where it is empirically observed that adversarial training does not always guarantee robust generalization \citep{schmidt2018adversarially}. Compared with inductive learning, transductive learning offers a new perspective on adversarial robustness that highlights how unlabeled adversarial test data can inform local robustness, which perhaps is easier to achieve than global robustness.

\section{Preliminaries} 
\label{sec:setup}

Let $\calX$ denote the instance space and $\calY=\set{\pm1}$. Let $\calH\subseteq \calY^\calX$ denote a hypothesis class and $\vc(\calH)$ denotes its VC dimension. Let $\calU:\calX\to 2^{\calX}$ denote an arbitrary perturbation set such that for each $x\in\calX$, $\calU(x)$ is non-empty. Denote by $\calU^{-1}$ the inverse image of $\calU$, where for each $z\in \calX$, $\calU^{-1}(z)=\set{x\in\calX: z\in\calU(x)}$. Observe that for any $x,z\in\calX$ it holds that $z\in\calU(x) \Leftrightarrow x\in \calU^{-1}(z)$. For an instance $x\in\calX$, $\calU^{-1}(\calU)(x)$ denotes the set of all \emph{natural} examples $\Tilde{x}$ that share some perturbation with $x$ according to $\calU$, i.e., $\calU^{-1}(\calU)(x)=\cup_{z\in\calU(x)} \calU^{-1}(z)=\set{\Tilde{x}\in\calX: \calU(x)\cap \calU(\Tilde{x})\neq \emptyset}$. For any sequence of labeled points $(\bx,\by) \in (\calX\times \calY)^n$, any sequence of adversarial perturbations $\bz\in \calX^n$, and any predictor $h:\calX\to \calY$ let $\err_{\bx,\by}(h) = \frac{1}{n} \sum_{i=1}^{n} \ind\set{h(x)\neq y}$ denote the standard $0$-$1$ error, and define:
\begin{align}
\label{eqn:risks}
    \Risk_{\calU^{-1}}(h;\bz,\by)&=\frac{1}{n}\sum_{i=1}^{n} \ind\set{\exists \Tilde{x}_i \in \calU^{-1}(z_i): h(\Tilde{x}_i)\neq y_i}.\\
    \Risk_{\calU^{-1}}(h;\bz)&=\frac{1}{n}\sum_{i=1}^{n}\ind\set{\exists \Tilde{x}_i\in \calU^{-1}(z_i): h( \Tilde{x}_i)\neq h(z_i)}.
\end{align}

Our transductive robust learning guarantees (presented in \prettyref{sec:main-results}) are in fact in terms of an adaptive complexity measure -- that is in general tighter than the VC dimension and takes into account the complexity of both $\calH$ and $\calU$ -- which we introduce next:

\begin{defn}[Relaxed Robust Shattering Dimension]
\label{def:robustshatter-dim-relax}
A sequence $z_1,\ldots,z_k \in \calX$ is said to be \emph{relaxed $\calU$-robustly shattered} by $\calH$ if $\forall y_1,\dots,y_k \in \set{\pm 1}: \exists x^{y_1}_1, \ldots, x^{y_k}_k \in \calX \text{ and } \exists h\in \calH$ such that $z_i\in \calU(x^{y_i}_i)$ and $h(\calU(x^{y_i}_i))=y_i \forall 1\leq i\leq k$. The \emph{relaxed $\calU$-robust shattering dimension} $\Rdim_{\calU}(\calH)$ is defined as the largest $k$ for which there exist $k$ points that are relaxed $\calU$-robustly shattered by $\calH$.
\end{defn}

The above complexity measure is inspired by the robust shattering dimension that was introduced by \citet{pmlr-v99-montasser19a} and shown to lower bound the sample complexity of robust learning in the inductive setting. The definition of $\Rdim_\calU(\calH)$ implies the following:

\begin{rem}
\label{rem:vcupper}
For any class $\calH$ and any perturbation set $\calU$, $\Rdim_\calU(\calH) \leq \vc(\calH)$. 
\end{rem}

\section{Main Results}
\label{sec:main-results}
We obtain strong robust learning guarantees against worst-case adversarial perturbations of \emph{both} the training data and the test data. Specifically, after training examples $(\bx,\by)\sim\calD^n$ and test examples $(\tbx,\tby)\sim\calD^n$ are drawn, an adversary perturbs both training and test examples by choosing adversarial perturbations $\bz\in\calU(\bx)$ and $\tbz \in \calU(\tbx)$. Our transductive learner observes as input $(\bz,\by)$ and $\tbz$, and outputs $\hat{h}(\tbz)\in \calY^n$ where $\hat{h}\in \Delta^{\calU}_{\calH}(\bz,\by,\tbz)$ defined as follows:
\begin{equation}
\label{eqn:learner-realizable}
\Delta^{\calU}_{\calH}(\bz,\by,\tbz) = \set{h\in \calH : \Risk_{\calU^{-1}}(h; \bz,\by) =0 \wedge \Risk_{\calU^{-1}}(h;\tbz)=0}~~~~\textsf{(Realizable, $\OPT_{\calU^{-1}(\calU)}=0$)}.
\end{equation}
\begin{equation}
\label{eqn:learner-agnostic}
\Delta^{\calU}_{\calH}(\bz,\by,\tbz) = \argmin_{h\in\calH} ~~\max\Big\{\Risk_{\calU^{-1}}(h; \bz,\by),~ \Risk_{\calU^{-1}}(h;\tbz)\Big\}~~~~\textsf{(Agnostic, $\OPT_{\calU^{-1}(\calU)}>0$)}.
\end{equation}

Our transductive learner simply asks for any predictor $h\in \calH$ that robustly and correctly labels the training examples $(\bz,\by)$ with respect to $\calU^{-1}$ and robustly labels the test examples $\bz$ with respect to $\calU^{-1}$. Observe that requiring robustness on $\bz$ and $\tbz$ with respect to $\calU^{-1}$ implies, by definition of $\calU^{-1}$, that the i.i.d. examples $\bx$ and $\tbx$ will be labeled in the same way as $\bz$ and $\tbz$, even though the learner does not observe $\bx$ and $\tbx$. This is the main insight that we rely on to obtain our transductive robust learning guarantees:

\begin{thm} [Realizable]
\label{thm:realizable-train-perturb}
For any $n\in\bbN$, $\delta > 0$, class $\calH$, perturbation set $\calU$, and distribution $\calD$ over $\calX\times \calY$ satisfying $\OPT_{\calU^{-1}(\calU)}= 0$:
\[
\underset{\substack{(\bx,\by)\sim \calD^n\\(\Tilde{\bx},\Tilde{\by})\sim\calD^n}}{\Pr} \insquare{ \forall \bz\in \calU(\bx),\forall \tbz \in \calU(\Tilde{\bx}), \forall \hat{h} \in \Delta^{\calU}_{\calH}(\bz,\by,\tbz):~ \err_{\tbz,\Tilde{\by}}(\hat{h}) \leq \epsilon } \geq 1-\delta,
\]
where $\epsilon=\frac{\Rdim_{\calU^{-1}}(\calH)\log(2n)+\log(1/\delta)}{n}\leq \frac{\vc(\calH)\log(2n)+\log(1/\delta)}{n}$.
\end{thm}

\begin{thm} [Agnostic]
\label{thm:agnostic-train-perturb}
For any $n\in\bbN$, $\delta>0$, class $\calH$, perturbation set $\calU$, and distribution $\calD$ over $\calX\times \calY$,
\[
\underset{\substack{(\bx,\by)\sim \calD^n\\(\tbx,\tby)\sim\calD^n}}{\Pr} \insquare{ \forall\bz\in\calU(\bx), \forall \tbz \in \calU(\tbx), \forall \hat{h} \in \Delta^{\calU}_{\calH}(\bz,\by,\tbz):~ \err_{(\tbz,\tby)}(\hat{h}) \leq \epsilon } \geq 1-\delta,
\]
where $\epsilon=\min\!\set{2\OPT_{\calU^{-1}(\calU)} + O\!\inparen{\sqrt{\frac{\vc(\calH)+\log(1/\delta)}{n}}},~ 3\OPT_{\calU^{-1}(\calU)} +O\!\inparen{\sqrt{\frac{\Rdim_{\calU^{-1}}(\calH)\ln(2n)+\ln(1/\delta)}{n}}}}$.
\end{thm}

\section{Transductive vs.~Inductive}
\label{sec:comparisons}
For purposes of the discussion below, let $\bbA_I: (\calX\times \calY)^*\to \calY^\calX$ denote an inductive learner and $\bbA_T:(\calX\times \calY)^n\times \calX^m\to\calY^m$ denote a trnasductive learner. The \emph{inductive} robust risk of $\bbA_I$ is defined as
\[
\IRisk^{n}_{\calU}(\bbA;\calD)= \Ex_{(\bx,\by)\sim\calD^n}\Risk_{\calU}(\bbA(\bx,\by);\calD) = \Ex_{(\bx,\by)\sim\calD^n} \Ex_{(x,y)\sim\calD} \sup_{z\in \calU(x)}\ind\set{\bbA(\bx,\by)(z)\neq y}.\]

For standard (non-robust) supervised learning, i.e., when $\calU(x)=\set{x}$, there isn't much difference between the transductive and inductive settings in terms of statistical performance---an observation which has been employed in designing and analyzing inductive learning algorithms by relying on the transductive setting \citep{vapnik:74}. We can always take an inductive learner $\bbA_I$ and use it transductively as $\bbA_T$ defined as
\begin{equation}
    \label{eqn:inductive-to-transductive}  \forall i\in[m]: \bbA_T(\bx,\by,\tbx)(\Tilde{x}_i)=\bbA_I(\bx,\by)(\Tilde{x}_i),
\end{equation}
and so $\TRisk^{n,m}(\bbA_T;\calD)\leq \Ex\insquare{\frac{1}{n}\sum_{i=1}^{n}\ind\set{\bbA_I(\bx,\by)(\Tilde{x}_i)\neq \Tilde{y}_i}}=\IRisk^n(\bbA_I;\calD)$.

In the other direction, given a transductive learner $\bbA_T$, if it's guarantee doesn't depend on the test set size $m$ (i.e., holds even when $m=1$), we can consider an inductive learner $\bbA_I$ that outputs a predictor which just runs the transductive learner at test-time defined as
\begin{equation}
    \label{eqn:transductive-to-inductive}
    \forall x\in \calX: \bbA_I(\bx,\by)(x)=\bbA_T(\bx,\by,x),
\end{equation}
ensuring $\IRisk^{n}(\bbA_I;\calD)=\TRisk^{n,1}(\bbA_T;\calD)$.

More generally, if the transductive learner $\bbA_T$ does rely on having multiple test examples, e.g., $m=n$ as in our case, we can randomly split the training set, using some of the training examples as test examples:
\begin{equation}
    \label{eqn:transductive-to-inductive-2}
     \forall x\in \calX: \bbA_I(\bx,\by)(x)=\bbA_T(\bx',\by',\bx''\cup\set{x})(x),
\end{equation}
where $\bx'$ and $\bx''$ are random disjoint subsets of $\bx$ of size $\floor{\frac{n}{2}}$ and $\floor{\frac{n}{2}}-1$, and $\bx''\cup \set{x}$ is a random permutation of the concatenation. This ensures
\begin{align*}
    \IRisk^{n}(\bbA_I;\calD)&= \Ex\insquare{\ind\set{\bbA_T(\bx',\by',\bx''\cup \set{x})(x)\neq y}}= \underset{\substack{(\bx,\by)\sim \calD^{n/2}\\(\tbx,\tby)\sim\calD^{n/2}}}{\Ex}\Ex_{i\sim {\rm Unif}[n/2]} \insquare{\ind\set{\bbA_T(\bx,\by,\tbx)(\Tilde{x}_i)\neq \Tilde{y}_i}}\\
    &= \underset{\substack{(\bx,\by)\sim \calD^{n/2}\\(\tbx,\tby)\sim\calD^{n/2}}}{\Ex} \frac{2}{n} \sum_{i=1}^{\frac n2} \ind\set{\bbA_T(\bx,\by,\tbx)(\Tilde{x}_i)\neq \Tilde{y}_i}=\TRisk^{\frac n2,\frac n2}(\bbA_T;\calD).
\end{align*}

Why is the robust setting different? We can still reduce transductive to inductive just the same. Given an inductive learner $\bbA_I$, the construction of \prettyref{eqn:inductive-to-transductive} is still valid, and we have $\TRisk^{n,m}_\calU(\bbA_T;\calD) \leq \IRisk^{n}_{\calU}(\bbA_I;\calD)$.

But what happens in the reverse direction? If transductive learner $\bbA_T$ doesn't rely on having multiple test examples, i.e., its guarantee doesn't depend on $m$ and is valid even if $m=1$, the construction in \prettyref{eqn:transductive-to-inductive} can still be used, and we have $\IRisk_\calU^{n}(\bbA_I;\calD)=\TRisk_\calU^{n,1}(\bbA_T;\calD)$.
This reduction has the potential of aiding in designing robust inductive learning methods, but it relies on the transductive method not depending on the number of test examples, or equivalently referred to as 1-point transductive learners (e.g., the one-inclusion graph prediction algorithm due to \citet{DBLP:journals/iandc/HausslerLW94}). Unfortunately, this is not the case for our transductive learner $\bbA_T$ (presented in \prettyref{sec:main-results}) which requires $m=n$.

Trying to apply the other reduction from \prettyref{eqn:transductive-to-inductive-2} and its analysis in the transductive setting, we would need to implement: $\bbA_T(\bx',\by',\bz''\cup\set{z})$. To apply our transductive learner $\bbA_T$, we would need not the subset of training points $\bx''$, but rather their perturbations $\bz''$. But how could we obtain this? This is not given to us. If $x \in \calU(x)$, we can try using $z''_i=x''_i$, i.e. $\bbA_T(\bx',\by',\bx''\cup \set{z})$, for which we get the following inductive error:
\begin{equation*}
    \IRisk_{\calU}^{n}(\bbA_I;\calD)= \underset{\substack{(\bx,\by)\sim \calD^{n/2}\\(\Tilde{\bx},\Tilde{\by})\sim\calD^{n/2}}}{\Ex}\insquare{ \Ex_{i\sim {\rm Unif}[n/2]} \sup_{\Tilde{z}_i\in\calU(\Tilde{x}_i)} \ind\set{\bbA_T(\bx,\by,\tbx_{1:i}, \Tilde{z}_i,\tbx_{i+1:\frac{n}{2}})(\Tilde{z}_i)\neq y_i} }.
\end{equation*}
But the right hand side here is only a (loose) upper bound on $\TRisk^{n/2,n/2}_\calU(\bbA_T;\calD)$. Specifically, in the above, the supremum representing the adversary, comes in after knowing which one of the $\frac n2$ points we are evaluating on. On the other hand, in a true transductive setting, i.e., in the definition of $\TRisk_\calU^{n/2,n/2}$, the adversary needs to commit to a perturbation that would be bad (for us) on all $\frac n2$ of the points.  In a sense, the fact that the adversary must perturb all points, and we can leverage knowledge of these perturbations, is what restricts the power of the adversary in the transductive setting, and allows us to better protect against adversarial attacks affecting many examples (we might still get a few examples wrong, but that's OK).

\paragraph{Proper vs.~Improper} Another issue where we see a difference between inductive and transductive adversarially robust learning is with regards to whether the learning can be proper. \citet{pmlr-v99-montasser19a} showed that learning some VC classes in the inductive setting necessarily requires improper learning. Specifically, there are classes $\calH$ with constant VC dimension that are not robustly PAC learnable with any inductive proper learner $\bbA_I:(\calX\times\calY)^*\to \calH$ which is constrained to outputting a predictor in $\calH$. Even in the case of robust realizability with respect to $\calU^{-1}(\calU)$, i.e., $\OPT_{\calU^{-1}(\calU)}=0$, we can still adapt the construction of \citet{pmlr-v99-montasser19a} to conclude that improper learning is needed in the inductive setting, whereas for transductive learning, our learner from \prettyref{sec:main-results} is proper. But this isn't surprising, and also in the standard (non-robust) setting we can expect differences in properness between transductive and inductive. 

We mentioned that any transductive non-robust learner can be transformed to an inductive learner, using the reduction in \prettyref{eqn:transductive-to-inductive-2}. But even if the transductive learner is proper, the resulting inductive learner is not. And furthermore, any improper transductive learner, whether non-robust or robust, can be transformed to a proper transductive learner. Specifically, for any transductive learner $\bbA_T$ and any input $(\bx,\by,\tbz)\in (\calX\times\calY)^n\times \calX^m$, we can project the labeling $\bbA_T(\bx,\by,\tbz)$ to the closest \emph{proper} labeling in the set $\Gamma_\calH(\tbz)=\set{(h(\tbz)): h\in \calH}$. In the realizable setting, when $\exists h\in\calH \text{ s.t. }\Risk_\calU(h;\calD)=0$, we are guaranteed that whenever $\bbA(\bx,\by,\tbz)$ has $\epsilon$ error then the \emph{proper} labeling has $2\epsilon$ error. In the agnostic setting, we are guaranteed that the \emph{proper} labeling will incur robust error at most $3\inf_{h\in\calH} \Risk_\calU(h;\calD) + 2\epsilon$ whenever $\bbA(\bx,\by,\tbz)$ has robust error of $\inf_{h\in\calH} \Risk_\calU(h;\calD) + \epsilon$. We therefore see that in the transductive setting, improperness can never buy a significant advantage, and we should not be surprised that learning that must be improper in the inductive setting can be proper in the transductive one.

\section{Proofs}

We start with proving a helpful lemma that extends the classic Sauer-Shelah-Perles lemma for the robust setting.

\begin{lem} [Sauer's lemma for $\Rdim_\calU(\calH)$]
\label{lem:sauer-robust}
For any class $\calH$, any perturbation set $\calU$, and any sequence of points $z_1,\dots, z_n\in \calX$,
{\small
\[\abs{\Pi^{\calU}_{\calH}(z_1,\dots,z_{n})} \triangleq \abs{\left\{ ( h(z_1),\ldots, h(z_{n}) ) \middle|
  \begin{subarray}{c} \exists x_{1},\dots, x_{n} \in \calX, \exists h \in \calH:\\
     z_i \in \calU(x_i) \wedge h(\calU(x_i))=h(z_i)\forall 1\leq i\leq n 
  \end{subarray}\right\}} \leq {n \choose \leq \Rdim_{\calU}(\calH) } \triangleq \sum_{i=0}^{\Rdim_{\calU}(\calH)} {n \choose i}. 
\]
}
\end{lem}

\begin{proof}
The proof will follow a standard argument that is used to prove Sauer-Shela-Perles lemma (see e.g., \cite{shalev2014understanding}). Specifically, it suffices to prove the following stronger claim:
\begin{equation}
\label{eqn:sauer-robust-strong}
    \abs{\Pi^{\calU}_{\calH}(z_1,\dots,z_{n})} \leq \abs{\set{ S\subseteq \set{z_1,\dots,z_n}: S \text{ is relaxed }\calU\text{-robustly shattered by }\calH }}.
\end{equation}
This is because 
\[\abs{\set{ S\subseteq \set{z_1,\dots,z_n}: S \text{ is relaxed }\calU\text{-robustly shattered by }\calH }} \leq {n \choose \leq \Rdim_{\calU}(\calH) }.\]
We will prove \prettyref{eqn:sauer-robust-strong} by induction on $n$. When $n=1$, both sides of \prettyref{eqn:sauer-robust-strong} either evaluate to 1 or 2 (the empty set is always considered to be relaxed $\calU$-robustly shattered by $\calH$). 
When $n>1$, assume that \prettyref{eqn:sauer-robust-strong} holds for sequences of length $k<n$. Let $C=\set{z_1,\dots,z_n}$ and $C'=\set{z_2,\dots,z_n}$. Consider the following two sets:
\[Y_0=\set{\inparen{y_2,\dots,y_n}: \inparen{+1,y_2,\dots,y_n}\in \Pi^{\calU}_{\calH}(z_1,\dots,z_{n}) \vee \inparen{-1,y_2,\dots,y_n} \in \Pi^{\calU}_{\calH}(z_1,\dots,z_{n})},\]
and
\[Y_1=\set{\inparen{y_2,\dots,y_n}: \inparen{+1,y_2,\dots,y_n}\in \Pi^{\calU}_{\calH}(z_1,\dots,z_{n}) \wedge \inparen{-1,y_2,\dots,y_n} \in \Pi^{\calU}_{\calH}(z_1,\dots,z_{n})}.\]
Observe that $\abs{\Pi^{\calU}_{\calH}(z_1,\dots,z_{n})}=\abs{Y_0}+\abs{Y_1}$. Additionally, note that by definition of $Y_0$, $Y_0\subseteq \Pi^{\calU}_{\calH}(z_2,\dots,z_{n})$. Thus, by the inductive assumption,
\begin{align*}
    \abs{Y_0}\leq \abs{\Pi^{\calU}_{\calH}(z_2,\dots,z_{n})} &\leq \abs{\set{ S\subseteq C': S \text{ is relaxed }\calU\text{-robustly shattered by }\calH }} \\ 
    &= \abs{\set{S\subseteq C: z_1\notin S \wedge S \text{ is relaxed }\calU\text{-robustly shattered by }\calH}}.
\end{align*}
Next, define $\calH' \subseteq \calH$ to be
{\small
\[\calH' = \set{h\in \calH: \exists h'\in \calH, \bx_{2:n},\tbx_{2:n} \in \calU^{-1}(\bz_{2:n}) \text{ s.t. } h(\calU(x_1))=-h'(\calU(x_1))\wedge h(\calU(\bx_{2:n}))=h'(\calU(\tbx_{2:n}))}.\]
}
Observe that if a set $S\subseteq C'$ is relaxed $\calU$-robustly shattered by $\calH'$, then $S\cup \set{z_1}$ is also relaxed $\calU$-robustly shattered by $\calH'$ and vice versa. Observe also that, by definition, $Y_1= \Pi^\calU_{\calH'}(z_2,\dots,z_n)$. By applying the inductive assumption on $\calH'$ and $C'$ we obtain that
\begin{align*}
\abs{Y_1}=\abs{\Pi^\calU_{\calH'}(z_2,\dots,z_n)} &\leq \abs{\set{ S\subseteq C': S \text{ is relaxed }\calU\text{-robustly shattered by }\calH' }}\\
&= \abs{\set{S\subseteq C': S\cup \set{z_1} \text{ is relaxed }\calU\text{-robustly shattered by }\calH'}}\\
&= \abs{\set{S\subseteq C: z_1\in S \wedge S \text{ is relaxed }\calU\text{-robustly shattered by }\calH'}}\\
&\leq \abs{\set{S\subseteq C: z_1\in S \wedge S \text{ is relaxed }\calU\text{-robustly shattered by }\calH}}.
\end{align*}
Overall, we have shown that
\begin{align*}
    \abs{\Pi^{\calU}_{\calH}(z_1,\dots,z_{n})}&=\abs{Y_0}+\abs{Y_1}\\
    &\leq \abs{\set{S\subseteq C: z_1\notin S \wedge S \text{ is relaxed }\calU\text{-robustly shattered by }\calH}} \\
    &~~~~~~~~+ \abs{\set{S\subseteq C: z_1\in S \wedge S \text{ is relaxed }\calU\text{-robustly shattered by }\calH}}\\
    &= \abs{\set{ S\subseteq \set{z_1,\dots,z_n}: S \text{ is relaxed }\calU\text{-robustly shattered by }\calH }},
\end{align*}
which concludes our proof.
\end{proof}

\subsection{Realizable Setting}

\begin{proof}[of \prettyref{thm:realizable-train-perturb}]
It suffices to show that 
$$\underset{\substack{(\bx,\by)\sim \calD^n\\(\Tilde{\bx},\Tilde{\by})\sim\calD^n}}{\Prob} \insquare{ \exists \bz \in \calU(\bx), \exists \tbz \in \calU(\Tilde{\bx}), \exists \hat{h} \in \Delta^{\calU}_{\calH}(\bz,\by,\tbz): \err_{\tbz,\Tilde{\by}}(\hat{h}) > \epsilon } \leq \delta.$$
Observe that since $\OPT_{\calU^{-1}(\calU)}=\inf_{h\in\calH} \Prob_{(x,y)\sim\calD} \insquare{\exists z\in\calU(x),\exists \Tilde{x}\in\calU^{-1}(z): h(\Tilde{x})\neq y} = 0$, it holds by definition of $\Delta_\calH^\calU$ (see \prettyref{eqn:learner-realizable}) that the set $\Delta_\calH^\calU(\bz,\by,\tbz)$ is non-empty with probability 1.

We will first start with a standard observation stating that sampling two iid sequences of length $n$, $(\bx,\by)\sim \calD^n$ and $(\Tilde{\bx},\Tilde{\by})\sim \calD^n$, is equivalent to sampling a single iid sequence of length $2n$, $(\bx,\by)\sim \calD^{2n}$, and then randomly splitting it into two sequences of length $n$ (using a permutation $\sigma$ of $\set{1,\dots,2n}$ sampled uniformly at random). Thus, it follows that
\begin{align*}
    \underset{\substack{(\bx,\by)\sim \calD^n\\(\Tilde{\bx},\Tilde{\by})\sim\calD^n}}{\Prob} &\insquare{ \exists \bz \in \calU(\bx), \exists \tbz \in \calU(\Tilde{\bx}), \exists \hat{h} \in \Delta^{\calU}_{\calH}(\bz,\by,\tbz): \err_{\tbz,\Tilde{\by}}(\hat{h}) > \epsilon }= \Ex_{(\bx,\by)\sim\calD^{2n}} \insquare{ \Prob_{\sigma} \insquare{E_{\sigma, \bz}\middle| (\bx,\by)} },
\end{align*}
where $\sigma$ is a permutation of $[2n]$ sampled uniformly at random and $E_{\sigma, \bz}$ is defined as:
\begin{equation*}
    E_{\sigma, \bz} = \set{\exists \bz_{\sigma(1:2n)} \in \calU(\bx_{\sigma(1:2n)}),\exists \hat{h} \in \Delta^\calU_\calH(\bz_{\sigma(1:n)}, \by_{\sigma(1:n)}, \bz_{\sigma(n+1:2n)}): \err_{\bz_{\sigma(n+1:2n)}, \by_{\sigma(n+1:2n)}}(\hat{h}) > \epsilon}.
\end{equation*}

\paragraph{High error on $\bz$'s implies high error on $\bx$'s.}It suffices to show that for any $(\bx,\by)\sim\calD^{2n}$ such that $\exists h^* \in \calH$ with $h^*(\calU^{-1}(\calU)(\bx))=\by$ (which occurs with probability one): $\Prob_{\sigma} \insquare{ E_{\sigma, \bz} \middle | (\bx,\by) } \leq \delta$. To this end, we will start by showing that the event $E_{\sigma, \bz}$ implies the following event $E_{\sigma, \bx}$:
\begin{equation*}
    E_{\sigma, \bx} = \set{\exists \bz_{\sigma(1:2n)} \in \calU(\bx_{\sigma(1:2n)}),\exists \hat{h} \in \Delta^\calU_\calH(\bz_{\sigma(1:n)}, \by_{\sigma(1:n)}, \bz_{\sigma(n+1:2n)}): \err_{\bx_{\sigma(n+1:2n)}, \by_{\sigma(n+1:2n)}}(\hat{h}) > \epsilon}.
\end{equation*}

In words, in case there are adversarial perturbations $\bz_{\sigma(1:2n)} \in \calU(\bx_{\sigma(1:2n)})$ and a predictor $\hat{h} \in \Delta^\calU_\calH(\bz_{\sigma(1:n)}, \by_{\sigma(1:n)}, \bz_{\sigma(n+1:2n)})$ with many mistakes on the adversarial perturbations:\\ $\err_{\bz_{\sigma(n+1:2n)}, \by_{\sigma(n+1:2n)}}(\hat{h}) > \epsilon$, then this implies that $\hat{h}$ makes many mistakes on the original non-adversarial test sequence: $\err_{\bx_{\sigma(n+1:2n)}, \by_{\sigma(n+1:2n)}}(\hat{h}) > \epsilon$. This is because for any $\bz_{\sigma(1:2n)} \in \calU(\bx_{\sigma(1:2n)})$, by definition of $\Delta^\calU_\calH$ (see \prettyref{eqn:learner-realizable}), any $\hat{h}\in \Delta^\calU_\calH(\bz_{\sigma(1:n)}, \by_{\sigma(1:n)}, \bz_{\sigma(n+1:2n)})$ robustly labels the perturbations $\bz_{\sigma(1:2n)}$: $\hat{h}(\calU^{-1}(\bz_{\sigma(1:2n)})) = \hat{h}(\bz_{\sigma(1:2n)})$. That is, 
\[\inparen{\forall 1 \leq i \leq 2n} \inparen{\forall \Tilde{x}\in \calU^{-1}(z_{\sigma(i)})}: \hat{h}(\Tilde{x}) = \hat{h}(z_{\sigma(i)}).\]
By definition of $\calU^{-1}$, it holds that $\bx_{\sigma(1:2n)}\in \calU^{-1}(\bz_{\sigma(1:2n)})$. Thus, it follows that $\hat{h}(\bz_{\sigma(n+1:2n)}) = \hat{h}(\bx_{\sigma(n+1:2n)})$, and therefore, event $E_{\sigma, \bz}$ implies event $E_{\sigma, \bx}$. 

\paragraph{Finite robust labelings on $\bx$'s} Based on the above, it suffices now to show that:\\
$\Prob_{\sigma} \insquare{ E_{\sigma, \bx} \middle | (\bx,\by) } \leq \delta$. To this end, we will show that for any permutation $\sigma$, any $\bz_{\sigma(1:2n)}\in \calU(\bx_{\sigma(1:2n)})$, and any $\hat{h}\in \Delta^\calU_\calH(\bz_{\sigma(1:n)}, \by_{\sigma(1:n)}, \bz_{\sigma(n+1:2n)})$ it holds that the labeling $\hat{h}(\bx_{\sigma(1:2n)})$ is included in a finite set of possible behaviors $\Pi^{\calU}_{\calH}$ defined on the entire sequence $\bx=(x_1,\dots,x_{2n})$ by:
\[\mathclap{ \Pi^{\calU}_{\calH}(x_1,\dots,x_{2n}) = \left\{ ( h(x_1), h(x_2), \ldots, h(x_{2n}) ) \middle|
  \begin{subarray}{c} \exists z_{1}\in \calU(x_1),\dots, z_{2n}\in \calU(x_{2n}), \\
    \exists h \in \calH: h(\calU^{-1}(z_i))=h(x_i)\forall 1\leq i\leq 2n 
  \end{subarray}\right\}}
\]

Consider an arbitrary permutation $\sigma$ and an arbitrary $\bz_{\sigma(1:2n)}\in \calU(\bx_{\sigma(1:2n)})$. For any $\hat{h}\in \Delta^\calU_\calH(\bz_{\sigma(1:n)}, \by_{\sigma(1:n)}, \bz_{\sigma(n+1:2n)})$, by definition of $\Delta^\calU_\calH$, it holds that $\hat{h}\inparen{\calU^{-1}(\bz_{\sigma(1:n)})}=\by_{\sigma(1:n)}$ and $\hat{h}(\calU^{-1}(\bz_{\sigma(n+1:2n)})) = \hat{h}(\bz_{\sigma(n+1:2n)})$. Therefore, $\bz_{\sigma(1:2n)} \in \calU(\bx_{\sigma(1:2n)})$ and the predictor $\hat{h}\in \calH$ are witnesses that satisfy the following:

\[\inparen{\forall 1 \leq i \leq 2n}: z_{\sigma(i)} \in \calU(x_{\sigma(i)}) \wedge \hat{h}(\calU^{-1}(z_{\sigma(i)}))=\hat{h}(x_{\sigma(i)}).\]

Thus, by definition of $\Pi^{\calU}_{\calH}$, it holds that $\hat{h}(\bx_{\sigma(1:2n)}) \in \Pi^{\calU}_{\calH}(x_1,\dots,x_{2n})$. This allows us to establish that the event $E_{\sigma,\bx}$ implies the event that there exists a labeling $\hat{h}(\bx_{\sigma(1:2n)}) \in \Pi^{\calU}_{\calH}(x_1,\dots,x_{2n})$ that achieves zero loss on the training examples $\err_{\bx_{\sigma(1:n)},\by_{\sigma(1:n)}}(\hat{h})=0$, but makes error more than $\epsilon$ on the test examples $\err_{\bx_{\sigma(n+1:2n)},\by_{\sigma(n+1:2n)}}(\hat{h})>\epsilon$. Specifically,
\begin{align*}
    \Prob_{\sigma} \insquare{ E_{\sigma, \bx} }
    &\leq \Prob_\sigma \insquare{\exists \hat{h}\in \Pi^{\calU}_{\calH}(x_1,\dots,x_{2n}): \err_{\bx_{\sigma(1:n)},\by_{\sigma(1:n)}}(\hat{h})=0 \wedge \err_{\bx_{\sigma(n+1:2n)},\by_{\sigma(n+1:2n)}}(\hat{h})>\epsilon}\\
    &\overset{(i)}{\leq} \abs{ \Pi^{\calU}_{\calH}(x_1,\dots,x_{2n}) } 2^{\ceil{-\epsilon n}} \overset{(ii)}{\leq} \inparen{2n}^{\Rdim_{\calU^{-1}}(\calH)} 2^{\ceil{-\epsilon n}},
\end{align*}
where inequality $(i)$ follows from applying a union bound over labelings $\hat{h}\in \Pi^{\calU}_{\calH}(x_1,\dots,x_{2n})$, and observing that for any such fixed $\hat{h}$:
\[\Prob_\sigma \insquare{ \err_{\bx_{\sigma(1:n)},\by_{\sigma(1:n)}}(\hat{h})=0 \wedge \err_{\bx_{\sigma(n+1:2n)},\by_{\sigma(n+1:2n)}}(\hat{h})>\epsilon } \leq 2^{-\ceil{\epsilon n}}. \]
To see this, suppose that $s=\sum_{i=1}^{2n} \ind\{\hat{h}(x_i)\neq y_i\}\geq \ceil{\epsilon n}$ (otherwise, the probability of the event above is zero). Now, when sampling a random permutation $\sigma$, the chance that all of the mistakes fall into the test split is at most $2^{-s}\leq 2^{-\ceil{\epsilon n}}$. Because if we pair the $s$ mistakes and any $s$ out of the $2n-s$ non-mistakes while fixing the remaining non-mistakes to be in the training split, then the chance that all the $s$ mistakes appear in the test split is at most $2^{-s}$. 

Finally, inequality $(ii)$ follows from applying Sauer's lemma on our introduced relaxed notion of robust shattering dimension (\prettyref{def:robustshatter-dim-relax}). Setting $\inparen{2n}^{\Rdim_{\calU^{-1}}(\calH)} 2^{\ceil{-\epsilon n}} \leq \delta$ and solving for $\epsilon$ yields the stated bound.
\end{proof}

\subsection{Agnostic Setting}

\begin{proof}[ of \prettyref{thm:agnostic-train-perturb}] 
Let $n\in\bbN$. For notational brevity, we write $\OPT = \OPT_{\calU^{-1}(\calU)}$.
We will assume that $\OPT$ (see \prettyref{eqn:opt}) is attained by some predictor $h^* \in \calH$\footnote{Otherwise, we can always choose a predictor $h^*\in\calH$ attaining $\OPT+\epsilon'$ for any small $\epsilon'>0$.}. For this fixed $h^*$, observe that by a standard Hoeffding bound, for $\epsilon_0=\sqrt{\frac{\ln(2/\delta)}{2n}}$, it holds that 
\begin{equation*}
    \underset{\substack{(\bx,\by)\sim \calD^n\\(\Tilde{\bx},\Tilde{\by})\sim\calD^n}}{\Prob} \insquare{ \inparen{\Risk_{\calU^{-1}(\calU)}(h^*; \bx,\by) \leq \OPT +\epsilon_0}  \wedge \inparen{\Risk_{\calU^{-1}(\calU)}(h^*; \tbx,\tby) \leq \OPT+\epsilon_0} } \geq 1-\delta.
\end{equation*}

By definition of $\Risk_{\calU^{-1}(\calU)}$ (see \prettyref{eqn:risks}), this implies that
\begin{equation*}
    \underset{\substack{(\bx,\by)\sim \calD^n\\(\tbx,\tby)\sim\calD^n}}{\Prob} \insquare{ \inparen{\forall \bz \in \calU(\bx): \Risk_{\calU^{-1}}(h^*; \bz,\by) \leq \OPT +\epsilon_0}  \wedge \inparen{\forall \tbz \in \calU(\tbx): \Risk_{\calU^{-1}}(h^*;\tbz) \leq \OPT+\epsilon_0} } \geq 1-\delta.
\end{equation*}

This implies that
\begin{equation*}
    \underset{\substack{(\bx,\by)\sim \calD^n\\(\tbx,\tby)\sim\calD^n}}{\Prob} \insquare{ \forall \bz\in \calU(\bx), \forall \tbz \in \calU(\tbx): \min_{h\in\calH} \max\set{\Risk_{\calU^{-1}}(h; \bz,\by), \Risk_{\calU^{-1}}(h;\tbz)} \leq \OPT +\epsilon_0 } \geq 1-\delta.
\end{equation*}

By \prettyref{eqn:learner-agnostic}, we have
{\small
\begin{equation*}
    \underset{\substack{(\bx,\by)\sim \calD^n\\(\tbx,\tby)\sim\calD^n}}{\Prob} \insquare{ \forall \bz \in \calU(\bx), \forall \tbz \in \calU(\tbx), \forall \hat{h} \in \Delta^\calU_\calH(\bz,\by,\tbz): \max\set{\Risk_{\calU^{-1}}(\hat{h}; \bz,\by), \Risk_{\calU^{-1}}(\hat{h};\tbz)} \leq \OPT +\epsilon_0 } \geq 1-\delta.
\end{equation*}}

That is, we have established that with high probability over the drawings of $(\bx,\by),(\tbx,\tby)\sim \calD^n$: for any adversarial perturbations $\bz\in \calU(\bx)$, $\tbz\in\calU(\tbx)$, and any predictor $\hat{h}\in \Delta^{\calU}_{\calH}(\bz,\by,\tbz)$:
\begin{enumerate}
    \item $\hat{h}$ achieves low robust error on the training examples: $\Risk_{\calU^{-1}}(\hat{h}; \bz,\by)\leq \OPT +\epsilon_0$. 
    \item $\hat{h}$ is robust (but not necessarily correct) on many of the test examples: $\Risk_{\calU^{-1}}(\hat{h};\tbz)\leq \OPT+\epsilon_0$.
\end{enumerate}

\paragraph{VC Guarantee} Next, to show that $\hat{h}$ achieves low error on the test examples $\tbz$, we will combine the properties above with a standard guarantee in the transductive setting from VC theory, which states that for $\epsilon=O\inparen{\sqrt{\frac{\vc(\calH)+\log(1/\delta)}{n}}}$:
\begin{equation*}
\label{eqn:uniform-transductive-agnostic}
    \underset{\substack{(\bx,\by)\sim \calD^n\\(\tbx,\tby)\sim\calD^n}}{\Prob} \insquare{\forall h\in \calH: \abs{ \err_{\bx,\by}(h) - \err_{\tbx,\tby}(h)} \leq \epsilon } \geq 1-\delta.
\end{equation*}

Thus, for $\epsilon=O\inparen{\sqrt{\frac{\vc(\calH)+\log(1/\delta)}{n}}}$:
\begin{equation*}
    \underset{\substack{(\bx,\by)\sim \calD^n\\(\tbx,\tby)\sim\calD^n}}{\Prob} \left[\begin{array}{lr}
        \forall \bz \in \calU(\bx), \forall \tbz \in \calU(\tbx), \forall \hat{h} \in \Delta^\calU_\calH(\bx,\by,\bz):\\ \max\set{\Risk_{\calU^{-1}}(\hat{h}; \bz,\by), \Risk_{\calU^{-1}}(\hat{h};\tbz)} \leq \OPT +\epsilon \wedge \abs{ \err_{\bx,\by}(\hat{h}) - \err_{\tbx,\tby}(\hat{h})} \leq \epsilon\\
        \end{array}\right] \geq 1-2\delta.
\end{equation*}

Finally, observe that for any predictor $\hat{h}\in \Delta^\calU_\calH(\bz,\by,\tbz)$ satisfying\\ $\max\set{\Risk_{\calU^{-1}}(\hat{h}; \bz,\by), \Risk_{\calU^{-1}}(\hat{h};\tbz)} \leq \OPT +\epsilon$ and $|\err_{\bx,\by}(\hat{h}) - \err_{\tbx,\tby}(\hat{h})| \leq \epsilon$, we can deduce that:
\begin{itemize}
    \item $\err_{\bx,\by}(\hat{h}) \leq \OPT+\epsilon$ (since $\Risk_{\calU^{-1}}(\hat{h}; \bz,\by) \leq \OPT + \epsilon$ and $\bx\in\calU^{-1}(\bz)$). Therefore, $\err_{\tbx,\tby}(\hat{h}) \leq \OPT+2\epsilon$ (since $|\err_{\bx,\by}(\hat{h}) - \err_{\tbx,\tby}(\hat{h})| \leq \epsilon$).
    \item Since $\err_{\tbx,\tby}(\hat{h}) \leq \OPT+2\epsilon$ and $\Risk_{\calU^{-1}}(\hat{h};\bz) \leq \OPT +\epsilon$, this implies that $\err_{\bz,\tby}(\hat{h})\leq 2\OPT+3\epsilon$.
\end{itemize}

\paragraph{A refined bound}  
We will show that 
\begin{equation*}
\label{eqn:refined-transductive-agnostic}
\begin{split}
    \underset{\substack{(\bx,\by)\sim \calD^n\\(\tbx,\tby)\sim\calD^n}}{\Prob} \biggl[&
        \forall \bz \in \calU(\bx), \forall \tbz \in \calU(\tbx), \forall \hat{h} \in \Delta^\calU_\calH(\bz,\by,\tbz):\\ &\max\set{\Risk_{\calU^{-1}}(\hat{h}; \bz,\by), \Risk_{\calU^{-1}}(\hat{h};\tbz)} \leq \OPT +\Tilde{\epsilon} \wedge \abs{ \err_{\bx,\by}(\hat{h}) - \err_{\tbx,\tby}(\hat{h})} \leq \Tilde{\epsilon} \biggr] \geq 1-\delta,
\end{split}
\end{equation*}
for a smaller $\Tilde{\epsilon}$ that scales with $\OPT$ and $\Rdim_{\calU^{-1}}(\calH)$ (instead of $\vc(\calH)$). To this end, it suffices to show that for any fixed $(\bx,\by)\sim \calD^{2n}$:
{\small
\begin{equation*}
    \underset{\sigma}{\Prob} \left[\begin{array}{lr}
        \forall \bz_{\sigma(1:2n)} \in \calU(\bx_{\sigma(1:2n)}), \forall \hat{h} \in \Delta^\calU_\calH(\bz_{\sigma(1:n)},\by_{\sigma(1:n)},\bz_{\sigma(n+1:2n)}):\\
         \max\set{\Risk_{\calU^{-1}}(\hat{h}; \bz_{\sigma(1:n)},\by_{\sigma(1:n)}), \Risk_{\calU^{-1}}(\hat{h};\bz_{\sigma(n+1,2n)})} \leq \OPT +\Tilde{\epsilon} \wedge \abs{ \err_{\bx,\by}(\hat{h}) - \err_{\tbx,\tby}(\hat{h})} \leq \Tilde{\epsilon}\\
        \end{array}\right] \geq 1-\delta,
\end{equation*}
}
where $\sigma$ is a permutation of $\set{1,2,3,\dots,2n}$ sampled uniformly at random. 

\noindent We will show that for any permutation $\sigma$, any $\bz_{\sigma(1:2n)}\in \calU(\bx_{\sigma(1:2n)})$, \\
and any $\hat{h}\in \Delta^\calU_\calH(\bz_{\sigma(1:n)}, \by_{\sigma(1:n)}, \bz_{\sigma(n+1:2n)})$ it holds that the labeling $\hat{h}(\bx_{\sigma(1:2n)})$ is included in a finite set of possible behaviors $\Pi^{\calU}_{\calH}$ defined on the entire sequence $\bx=(x_1,\dots,x_{2n})$ by:
\[\mathclap{ \Pi^{\calU}_{\calH}(x_1,\dots,x_{2n}) = \left\{ ( h(x_1), h(x_2), \ldots, h(x_{2n}) ) \middle|
  \begin{subarray}{c} \exists I\subseteq [2n], \abs{I}\geq (1-\OPT-\epsilon_0)2n,\\
  \forall i\in I, \exists z_{i}\in \calU(x_i), \\
    \exists h \in \calH: h(\calU^{-1}(z_i))=h(x_i) \forall i \in I
  \end{subarray}\right\}}
\]
When it holds that $\Risk_{\calU^{-1}}(\hat{h}; \bz_{\sigma(1:n)},\by_{\sigma(1:n)})\leq \OPT +\epsilon_0$ and $\Risk_{\calU^{-1}}(\hat{h}; \bz_{\sigma(n+1:2n)})\leq \OPT +\epsilon_0$, by definition of $\Risk_{\calU^{-1}}$ (see \prettyref{eqn:risks}), it follows that the predictor $\hat{h}\in \calH$ and $\bz_{\sigma(1:2n)} \in \calU(\bx_{\sigma(1:2n)})$ are witnesses that satisfy the following:
\[\inparen{\exists I\subseteq [2n], \abs{I} \geq (1-\OPT-\epsilon)2n} \inparen{\forall i\in I}: z_{\sigma(i)} \in \calU(x_{\sigma(i)}) \wedge \hat{h}(\calU^{-1}(z_{\sigma(i)}))=\hat{h}(x_{\sigma(i)}).\]

Thus, by definition of $\Pi^{\calU}_{\calH}$, it holds that $\hat{h}(\bx_{\sigma(1:2n)}) \in \Pi^{\calU}_{\calH}(x_1,\dots,x_{2n})$. Then, observe that for any $\Tilde{\epsilon} > 0$
\begin{align*}
\Prob_\sigma &\insquare{\exists \hat{h}\in \Pi^{\calU}_{\calH}(x_1,\dots,x_{2n}): \inabs{ \err_{\bx_{\sigma(1:n)},\by_{\sigma(1:n)}}(\hat{h}) - \err_{\bx_{\sigma(n+1:2n)},\by_{\sigma(n+1:2n)}}(\hat{h})} > \Tilde{\epsilon}}\\
    &\overset{(i)}{\leq} \abs{ \Pi^{\calU}_{\calH}(x_1,\dots,x_{2n}) } e^{-\Tilde{\epsilon}^2 n} \overset{(ii)}{\leq} {2n \choose \leq (\OPT+\epsilon_0)2n}{(1-\OPT-\epsilon_0)2n \choose \leq \Rdim_{\calU^{-1}}(\calH)} e^{-\Tilde{\epsilon}^2n}\\ &\overset{(iii)}{\leq} 2^{H(\OPT+\epsilon_0)2n}\inparen{(1-\OPT-\epsilon_0)2n}^{\Rdim_{\calU^{-1}}(\calH)} e^{-\Tilde{\epsilon}^2 n},
\end{align*}
where inequality $(i)$ follows from applying a union bound over labelings $\hat{h}\in \Pi^{\calU}_{\calH}(x_1,\dots,x_{2n})$, and observing that for any such fixed $\hat{h}$:
\[\Prob_\sigma \insquare{\inabs{ \err_{\bx_{\sigma(1:n)},\by_{\sigma(1:n)}}(\hat{h}) - \err_{\bx_{\sigma(n+1:2n)},\by_{\sigma(n+1:2n)}}(\hat{h})} > \Tilde{\epsilon}} \leq e^{-\Tilde{\epsilon}^2n}.\]
Inequality $(ii)$ follows from the definition of $\Pi^\calU_\calH$ and applying Sauer's lemma on our introduced relaxed notion of robust shattering dimension (\prettyref{lem:sauer-robust}). Inequality $(iii)$ follows from bounds on the binomial coefficients, where $H$ is the entropy function. 

Setting $2^{H(\OPT+\epsilon_0)2n}\inparen{(1-\OPT-\epsilon_0)2n}^{\Rdim_{\calU^{-1}}(\calH)} e^{-\Tilde{\epsilon}^2 n}$ less than $\frac{\delta}{2}$ and solving for $\Tilde{\epsilon}$ yields:
\begin{align*}
    \Tilde{\epsilon} &\leq \sqrt{2\ln(2)H(\OPT+\epsilon_0)+\frac{\Rdim_{\calU^{-1}}(\calH)\ln\inparen{(1-\OPT-\epsilon_0)2n}+\ln(1/\delta)}{n}}\\
    &\leq \sqrt{2\ln(2)H(\OPT+\epsilon_0)} + \sqrt{\frac{\Rdim_{\calU^{-1}}(\calH)\ln\inparen{(1-\OPT-\epsilon_0)2n}+\ln(1/\delta)}{n}}\\
    &\leq \OPT+\epsilon_0 + \sqrt{\frac{\Rdim_{\calU^{-1}}(\calH)\ln\inparen{2n}+\ln(1/\delta)}{n}}
\end{align*}
Combining both events from above, we get that $\err_{\tbz,\tby}(\hat{h}) \leq \err_{\tbx,\tby}(\hat{h}) + \Risk_{\calU^{-1}}(\hat{h};\tbz) \leq \err_{\bx,\by}(\hat{h}) +\Tilde{\epsilon}+\Risk_{\calU^{-1}}(\hat{h};\tbz)\leq \OPT +\epsilon_0 +\Tilde{\epsilon}+\OPT+\epsilon_0=2\OPT+2\epsilon_0+\Tilde{\epsilon}\leq 3\OPT+3\epsilon_0+\sqrt{\frac{\Rdim_{\calU^{-1}}(\calH)\ln\inparen{2n}+\ln(1/\delta)}{n}}$.
\end{proof}

\section{Discussion}
\label{sec:dis}

\paragraph{Related Work} Adversarially robust learning has been mainly studied in the inductive setting \citep[see e.g.,][]{schmidt2018adversarially,DBLP:conf/nips/CullinaBM18,khim2018adversarial,bubeck2019adversarial,DBLP:conf/icml/YinRB19,pmlr-v99-montasser19a}. This includes studying what learning rules should be used for robust learning and how much training data is needed to guarantee low robust error. It is now known that any hypothesis class $\calH$ with finite VC dimension is robustly learnable, though sometimes improper learning is necessary and the best known upper bound on the robust error is exponential in the VC dimension \citep{pmlr-v99-montasser19a}. 

In transductive learning, the learner is given unlabeled test examples to classify all at once or in batches, rather than individually \citep*{vapnik1998statistical}. Without robustness guarantees, it is known that $\ERM$ is nearly minimax optimal in the transductive setting \citep{vapnik:74,blumer:89,DBLP:journals/corr/TolstikhinL16}. In particular, additional unlabeled test data does not offer any help from a minimax perspective. More recently, \citep{DBLP:conf/nips/GoldwasserKKM20} gave a transductive learning algorithm that takes as input labeled training examples from a distribution $\calD$ and \emph{arbitrary} unlabeled test examples (chosen by an unbounded adversary, not necessarily according to perturbation set $\calU$). For classes $\calH$ of bounded VC dimension, their algorithm guarantees low error rate on the test examples but it might \emph{abstain} from classifying some (or perhaps even all) of them. This is different from the guarantees we present in this work, where we restrict the adversary to choose from a perturbation set $\calU$ but we do \emph{not} abstain from classifying.

On the empirical side, recently \citep{DBLP:conf/icml/WuYW20} proposed a method that leverages unlabeled test data for adversarial robustness in the context of deep neural networks. However, \citep{DBLP:journals/corr/abs-2106-08387} later proposed an empirical attack that breaks their defense. Furthermore, \citep{DBLP:journals/corr/abs-2106-08387} proposed another empirical transductive defense but with no theoretical guarantees. In semi-supervised learning, recent work \citep{DBLP:conf/nips/AlayracUHFSK19,DBLP:conf/nips/CarmonRSDL19} has shown that (non-adversarial) unlabeled test data can improve adversarially robust generalization in practice, and there is also theoretical work quantifying the benefit of unlabeled data for robust generalization \citep{DBLP:conf/icml/AshtianiPU20}.

\paragraph{Open Problems} Can we design transductive learners that compete with $\OPT_\calU$ instead of \\
$\OPT_{\calU^{-1}(\calU)}$? We note that this will likely require more sophistication, in the sense that we can construct classes $\calH$ with $\vc(\calH)=1$ (similar construction to \cite{pmlr-v99-montasser19a}) and distributions $\calD$ where $\OPT_\calU=0$ but $\OPT_{\calU^{-1}(\calU)}=1$, and moreover, our simple transductive learner fails and finding a robust labeling on the test examples no longer suffices. 

At the expense of competing with $\OPT_{\calU^{-1}(\calU)}$, can we obtain stronger robust learning guarantees in the inductive setting, similar to the transductive guarantees established in this work? As we discuss in \prettyref{sec:comparisons}, we can not obtain such guarantees by directly reducing to the transductive problem, and we need improper learning because proper learning will not work.

\acks{This work was done as part of the NSF/Simons sponsored Collaboration on the Theoretical Foundations of Deep Learning (\url{https://deepfoundations.ai/}). This work was supported in part by DARPA under cooperative agreement HR00112020003.\footnote{The views expressed in this work do not necessarily reflect the position or the policy of the Government and no official endorsement should be inferred. Approved for public release; distribution is unlimited.}}

\bibliography{learning}

\end{document}